%% file: main.tex
\documentclass{article}

\newif\ifarxiv
\arxivtrue  

\usepackage{iclr2026_conference,times}
\usepackage{amsmath,amssymb,amsthm}
\usepackage{graphicx}
\usepackage{float}
\usepackage{placeins}
\usepackage{booktabs}
\usepackage{hyperref}
\usepackage{url}

\ifarxiv\iclrfinalcopy\fi

\theoremstyle{definition}

\theoremstyle{plain}
\newtheorem{theorem}{Theorem}
\newtheorem{proposition}{Proposition}

\newcommand{\R}{\mathbb{R}}
\newcommand{\softmax}{\mathrm{softmax}}
\newcommand{\inner}[2]{\langle #1, #2 \rangle}

\title{All Routes Lead to Collapse}

\author{K.~R.~Balasubramanian\\
Independent Research}

\begin{document}
\maketitle
\ifarxiv\lhead{Preprint}\fi

\begin{abstract}
Attention sinks, representation collapse, and norm stratification are treated as
transformer-specific pathologies. We show they are not specific to attention:
they are what \emph{content-based routing} does under a fixed similarity metric.
We give a reframing identity: softmax attention is Boltzmann-weighted
aggregation over Euclidean distances with constant key norms, so its score omits
a $-\|k\|^2$ term and is blind to key magnitude. This predicts that any router
whose metric is ill-matched to its representations should compensate, by
concentrating its routing and collapsing the routed representations. We test it on
routers that score and aggregate over different axes: softmax attention over tokens
(nine pretrained transformers), graph attention over nodes, a selective state-space
model and a recurrent mixer over time, and learned residuals over depth. All develop
the same signature, and two within-model ablations show it is \emph{caused} by the
routing mechanism rather than by incidental dynamics. The \emph{form} is contingent,
set by the strength of the positional brake each router carries alongside its content
score; we sweep that brake and move the onset across its whole range. The
\emph{mechanism} is not contingent, and it does not require norm stratification: a
router with norm-normalized keys concentrates just the same. We do not claim these
models implement Riemannian geometry; the geometric view is a diagnostic that names
the inadequacy of the flat, norm-blind metric.
\end{abstract}

\input{sections/01_introduction}
\input{sections/02_background}
\input{sections/03_motivation}
\input{sections/04_transformers}
\input{sections/05_crossarch}
\input{sections/06_hypothesis}
\input{sections/07_discussion}
\input{sections/08_conclusion}

\bibliography{references}
\bibliographystyle{iclr2026_conference}

\appendix
\renewcommand{\topfraction}{0.95}
\renewcommand{\bottomfraction}{0.95}
\renewcommand{\textfraction}{0.05}
\renewcommand{\floatpagefraction}{0.70}
\input{sections/appendix_theory}
\input{sections/appendix_methods}
\input{sections/appendix_data}
\input{sections/appendix_convergent}
\input{sections/appendix_taxonomy}

\end{document}

%% file: sections/01_introduction.tex
\section{Introduction}
\label{sec:introduction}

Trained transformers develop a cluster of striking regularities. A few tokens absorb most of
the attention \citep{xiao2024streaming}, the hidden representations collapse toward a low-rank
subspace with depth \citep{dong2021attention}, and the key and query norms stratify rather than
staying uniform. These are usually read as pathologies of attention, diagnosed and patched one
model and one mechanism at a time. We argue they are not pathologies of attention at all. They
are what content-based routing does under a fixed, norm-blind similarity metric, and they appear
in any router that shares that metric, whatever it routes over.\ifarxiv\footnote{Code and data: \url{https://github.com/parzi-val/all-routes-lead-to-collapse}}\fi

The starting point is a reframing identity. Softmax attention can be written as a
Boltzmann-weighted aggregation over Euclidean distances between queries and keys, in which the
score drops a $-\|k\|^2$ term and so cannot see key magnitude. The metric the router uses is flat
and blind to norm. A router whose metric is ill-matched to its representations has to compensate,
and the compensation has a shape: the routing concentrates, the routed representations collapse,
and their norms stratify. The identity is a sufficient condition, not a claim that any model
implements geometry; it tells us where to look.

We look in five places. Across nine pretrained transformers the signature is present and
quantitatively clear against matched null baselines. It then appears, unchanged in kind, in four
routers that are not standard token attention: graph attention on heterophilic graphs, a
selective state-space model with no explicit attention, a recurrent mixer over time, and
attention residuals that route over depth. Where the routing weights can be reconstructed from
quantities we hold fixed, we ablate them and find the concentration is caused by the routing
mechanism rather than by incidental dynamics. What differs across architectures is the form, the
onset depth and the strength and the particular subspace, set largely by the strength of the
positional brake each router carries alongside its content score; we sweep that brake in two
architectures and watch the form move with it. What does not differ is the mechanism.

Our contributions are:
\begin{itemize}
\item A reframing identity (Section~\ref{sec:motivation}) that exposes the norm-blind metric
inside softmax routing and predicts a compensation signature.
\item Measurements of that signature across nine transformers against null baselines
(Section~\ref{sec:transformers}), and across four non-standard routers spanning graphs, time, and
depth, including two within-model causal ablations (Section~\ref{sec:crossarch}).
\item A hypothesis (Section~\ref{sec:hypothesis}) that separates the invariant mechanism from the
contingent form, names the positional brake as what sets the form, and shows with a
normalized-key router that norm stratification is one compensation rather than the cause.
\end{itemize}
The routes differ; the destination does not.

%% file: sections/02_background.tex
\section{Background and related work}
\label{sec:background}

\paragraph{Three pathologies, studied separately.}
Three phenomena recur in analyses of trained transformers. Attention sinks: a small number
of positions, often the first token, absorb a large share of the attention mass
\citep{xiao2024streaming}. Rank collapse: with depth the hidden representations lose
effective dimensionality and drift toward a shared subspace, an effect that pure attention
drives doubly exponentially \citep{dong2021attention}. And norm stratification: key and
query norms spread out rather than staying uniform. These are usually studied one at a time
and as properties of attention. We read them as three faces of one routing mechanism.

\paragraph{Architectures as routers.}
The models we measure share a structure: each scores a set of sources and aggregates them by
the scores. Softmax attention \citep{vaswani2017} scores tokens by a query-key dot product.
Graph attention \citep{velickovic2018gat} scores neighbors in a graph. Selective state-space
models such as Mamba \citep{gu2023mamba} have no explicit attention, but their selective scan
unrolls into a data-controlled operator of the same form, the hidden attention of
\citet{ali2024hidden}. RWKV \citep{peng2023rwkv} mixes over time with a decayed softmax.
Attention residuals \citep{kimi2026attnres} route over depth rather than tokens. We treat all
of these as content-based routers and ask what their shared metric does.

\paragraph{Oversmoothing, the graph-specific account.}
Graph neural networks carry their own collapse story: repeated neighborhood aggregation washes
node features toward a common value, independent of any learned attention
\citep{li2018deeper,oono2020graph}. Because oversmoothing reaches low rank by a route that is
not our mechanism, the graph case needs a control that separates learned-attention collapse
from generic smoothing, which Section~\ref{sec:crossarch} supplies.

\paragraph{The metric view.}
Our motivation rests on reading softmax as a maximum-entropy, or Boltzmann, weighting
\citep{jaynes1957}, which makes the underlying similarity metric explicit and exposes its
blindness to key magnitude. Section~\ref{sec:motivation} develops this identity; the rest of
the paper tests what it predicts.

%% file: sections/03_motivation.tex
\section{Attention as distance-based routing}
\label{sec:motivation}

This section makes precise the sense in which standard attention is routing over a
fixed geometry, and isolates the one assumption in that view with empirical content.
The result is an identity, not a contribution; it tells us what to measure.

\paragraph{Content-based routing.}
We call a layer a \emph{content-based router} if it produces each output as a
convex or conic combination of value vectors, with weights computed from a
similarity score between a query and a set of keys. Softmax attention is the
canonical instance, but the definition is deliberately architecture-agnostic:
graph attention scores a node against its neighbors, and a selective state-space
model scores the current position against its own past through an input-dependent
gate (Section~\ref{sec:crossarch}). Each such router carries a fixed
\emph{similarity metric} (the functional form of its score) together with a
learned representation that the metric acts on. The question this paper asks is
what happens when the two are mismatched.

\paragraph{The distance view.}
Fix a query $q$ and keys $\{k_i\}_{i=1}^n$ in $\R^d$. A distance-based router that
maximizes entropy at a fixed expected squared distance from $q$ assigns the Boltzmann
weights
\begin{equation}
\alpha_i \;=\; \frac{\exp\!\big(-\beta\, D(q,k_i)^2\big)}
{\sum_j \exp\!\big(-\beta\, D(q,k_j)^2\big)},
\label{eq:boltzmann}
\end{equation}
which are the unique maximum-entropy weights for that constraint \citep{jaynes1957}.
The Boltzmann form is therefore \emph{derived}, not posited; the KKT derivation is in
Appendix~\ref{app:theory} (Proposition~\ref{prop:maxent}).

\begin{theorem}
\label{thm:identity}
Suppose the keys lie on a smooth manifold whose metric is conformally flat with a
constant conformal factor (A1 through A3), so the geodesic distance is Euclidean up to a
constant, $D(q,k_i)^2 = \Omega^2\|q-k_i\|^2$, and suppose the keys are hyperspherical,
$\|k_i\| = c$ (A4). Then with $2\beta\Omega^2 = 1/\sqrt{d}$ the maximum-entropy weighting
\eqref{eq:boltzmann} equals scaled dot-product attention,
$\alpha_i = \softmax_i\!\big(\inner{q}{k_i}/\sqrt{d}\big)$.
\end{theorem}

\begin{proof}
Expand $\|q-k_i\|^2 = \|q\|^2 - 2\inner{q}{k_i} + \|k_i\|^2$. Substituting into
\eqref{eq:boltzmann}, the factor $\exp(-\beta\Omega^2\|q\|^2)$ is constant in $i$
and cancels in the normalization. The factor $\exp(-\beta\Omega^2\|k_i\|^2)$ is
constant in $i$ precisely because $\|k_i\|=c$, and likewise cancels. What remains
is $\alpha_i \propto \exp(2\beta\Omega^2\inner{q}{k_i})$; setting
$2\beta\Omega^2 = 1/\sqrt{d}$ gives the claim.
\end{proof}

\paragraph{What is actually assumed.}
Of the four assumptions, only A4 carries empirical weight. A1 through A3 are a geometric lens:
they flatten and uniformly scale the metric so $D$ is Euclidean up to a constant, and we
use them to name the metric the router commits to, not to claim that a trained network
performs geometry on a manifold. A4, that the keys are hyperspherical, is exactly what
cancels the $\|k_i\|^2$ term: standard attention scores with the bare inner product
$\inner{q}{k_i}$, the squared-distance score with that term removed, so removing it assumes
key norm does not vary. In this precise sense the dot-product score is \emph{blind to key
magnitude}.\ifarxiv{} Appendix~\ref{app:theory} states the assumptions, derives the Boltzmann form, and proves the identity in full.\fi

The theorem gives sufficiency, not necessity. Softmax attention is
\emph{representable as} flat-Euclidean Boltzmann routing with norm-uniform keys,
but it does not follow that a transformer assumes a manifold or performs
geometry, and we make no such claim. What is falsifiable is narrow: if key norms
are not constant, the dot-product score differs from the distance-based score by
the omitted, position-varying quantity $\|k_i\|^2$.

\paragraph{The inversion, and the prediction it yields.}
Read in reverse, Theorem~\ref{thm:identity} says standard attention commits to a flat
metric that ignores key magnitude, and whether that metric is adequate is an empirical
question about the representations the network learns. If the keys concentrate on a
low-dimensional, curved subset of $\R^d$, the flat metric over-counts directions the data
does not use and the norm-blindness discards magnitude the task may need. A router cannot
change its metric, but it can change the representation the metric sees. The hypothesis of
this paper is that it does, in three measurable forms: collapse toward a low-dimensional
subset, concentration of routing onto a few positions, and the use of norm as a control
signal. Sections~\ref{sec:transformers} and~\ref{sec:crossarch} test this; the key-norm
assumption is the first thing we check, and it fails in every layer of every model.

\paragraph{A correction outside the bilinear family.}
The diagnosis also names its own intervention: restoring the dropped term gives a
learned-metric score $-(q-k)^\top M (q-k)$ whose query-independent penalty $-k^\top M k$
no reparameterization of the query and key projections can reproduce, since a bilinear
score $q^\top A k$ vanishes at $q=0$ while this term does not. It is a new degree of
freedom, not a re-weighting of the existing one (Proposition~\ref{prop:bilinear}). Whether intervening on the metric
reduces the compensation we document is left to companion work; we note it only to show
that the geometric view points to interventions as well as measurements.

%% file: sections/04_transformers.tex
\section{The compensation signature in transformers}
\label{sec:transformers}

We measure the three quantities named in Section~\ref{sec:motivation} across nine
pretrained transformers: GPT-2 small, medium, large, and XL, and Pythia 160M,
410M, 1B, 1.4B, and 2.8B. All measurements use 150 sequences of length 128 from
the WikiText-103 validation split, in fp32, with keys taken before and after the
rotary transform for the Pythia models. Table~\ref{tab:transformer-summary}
summarizes the run; the three subsections below read it column by column.

\input{tables/tab_transformer_summary}

\subsection{Key norms are not uniform}
The one assumption with empirical content in Theorem~\ref{thm:identity} is that
keys share a common norm. They do not. The depth-averaged coefficient of
variation of key norms (column norm-CV) is well above zero in every model, and
the smallest per-layer value we observe anywhere is $0.114$, above the largest
isotropic-Gaussian baseline of $1/\sqrt{2d_h} = 0.088$. The violation is not
confined to a few layers or heads: it holds in every layer of every model. Key
norm carries information, and the dot-product score discards it.

\subsection{Key geometry is low-rank}
The keys also occupy far fewer effective dimensions than the space they live in.
Column VE@8 reports the rank-8 variance explained of the double-centered key
distance matrix, averaged over depth. To judge whether this is remarkable we
compare against two null baselines. The Gaussian null draws i.i.d.\ keys with the
observed mean and variance; the shuffle null permutes token positions within each
sequence, preserving the key-vector marginal distribution while destroying
token-order structure. Real keys clear both nulls in all nine models, with the depth
profile and the decomposition shown in Appendix~\ref{app:results},
Figure~\ref{fig:collapse-nulls}.

The Gaussian null's value is not arbitrary. The double-centered squared-distance matrix has
the same nonzero spectrum (up to a factor) as the centered Gram matrix, which for i.i.d.\
Gaussian keys is Wishart, so its limiting VE@8 follows the Marchenko-Pastur law and depends
only on the aspect ratio $c = n/d_h$ for $n=128$ tokens, not on the scale of the keys. This
shows in the data: the null is identical at $0.294$ across the six models with $d_h=64$, then
falls as $d_h$ grows, to $0.262$, $0.212$, and $0.160$ at $d_h = 80, 128, 256$. It is a
property of dimensionality, which keeps the low-rank claim honest: VE@8 describes the learned
key cloud under the Euclidean metric, not the curvature of any latent manifold.

The shuffle null is the more informative comparison.
We split the rank-8 collapse into a marginal
effect (shuffle minus Gauss), which the key-vector distribution alone produces,
and a content effect (real minus shuffle), which requires token-order structure.
The content effect grows with scale in both families, from $0.260$ to $0.322$
across GPT-2 and from $0.066$ to $0.214$ across Pythia. The smallest model,
Pythia 160M, is the cautionary case: almost all of its apparent collapse
($0.819$ of $0.885$) is reproduced by the shuffle null, so its low rank is mostly
a property of the key marginals rather than learned token structure. This caveat
disappears at scale.

\subsection{Routing concentrates early, but completes at no fixed depth}
Attention concentration sets in early. We fix the onset threshold at $0.20$ across
architectures, well above the uniform share $1/128$ and the nulls but below the peak of
even the weakest router (the recurrent models of Section~\ref{sec:crossarch} peak near
$0.22$ to $0.31$): a single threshold that every router crosses is what makes onset depth
comparable, and a higher bar would leave the weakest with none to report. The onset column
reports the first layer whose mean max-attention-share exceeds this threshold; across
the nine transformers it lands between 9\% and 25\% of depth. There is no single onset
depth, and no fixed completion depth either.

What varies is the completion dynamics. The Pythia models lock in quickly after
onset and do so earlier as they scale, with Pythia 2.8B crossing at 9\% depth.
The larger GPT-2 models do the opposite: GPT-2 XL first exceeds the $0.20$ onset
at 17\% depth but does not pass a $0.50$ share until 60\% depth, and both GPT-2
large and XL drop their concentration sharply in the final layer as routing gives
way to the readout. The signature is the same across the family; its timing is
not.

\paragraph{On the correlation between norm and concentration.}
The last column reports the Pearson correlation between key-norm CV and
max-attention-share across (layer, head) pairs. An earlier draft reported this correlation
reversing sign with scale and called it the sharpest finding here; it does not replicate.
In the standardized nine-model set every correlation is positive, and only the magnitude
changes: in GPT-2 it is strong and grows with scale ($0.37$ to $0.70$); in Pythia it is weak
and decays toward zero, not significant at 410M. The strength varies, the sign does not.

%% file: tables/tab_transformer_summary.tex
\begin{table}[t]
\centering
\small
\begin{tabular}{lrrrrrrrr}
\toprule
Model & $L$ & $d_h$ & norm-CV & VE@8 & Gauss & shuffle & onset & $r_{\mathrm{cv,sh}}$ \\
\midrule
GPT-2 small & 12 & 64 & 0.151 & 0.774 & 0.294 & 0.514 & L2 (17\%) & 0.37 \\
GPT-2 medium & 24 & 64 & 0.210 & 0.758 & 0.294 & 0.471 & L4 (17\%) & 0.45 \\
GPT-2 large & 36 & 64 & 0.249 & 0.732 & 0.294 & 0.422 & L8 (22\%) & 0.59 \\
GPT-2 XL & 48 & 64 & 0.307 & 0.737 & 0.294 & 0.415 & L8 (17\%) & 0.70 \\
\midrule
Pythia 160M & 12 & 64 & 0.258 & 0.885 & 0.294 & 0.819 & L2 (17\%) & 0.28 \\
Pythia 410M & 24 & 64 & 0.260 & 0.783 & 0.294 & 0.636 & L6 (25\%) & 0.05$^{\dagger}$ \\
Pythia 1B & 16 & 256 & 0.254 & 0.572 & 0.160 & 0.370 & L4 (25\%) & 0.18 \\
Pythia 1.4B & 24 & 128 & 0.289 & 0.733 & 0.212 & 0.554 & L3 (12\%) & 0.19 \\
Pythia 2.8B & 32 & 80 & 0.286 & 0.753 & 0.262 & 0.539 & L3 (9\%) & 0.07 \\
\bottomrule
\end{tabular}
\caption{Per-model summary on WikiText-103 ($N{=}150$, seq.\ len.\ 128). norm-CV is the depth-averaged key-norm coefficient of variation (A4 violation); every value exceeds its isotropic-Gaussian baseline ($1/\sqrt{2d_h}$). VE@8 is the depth-averaged rank-8 variance explained of the key distance matrix (pre-RoPE for Pythia); Gauss and shuffle are the matched null baselines. onset is the first layer whose mean max-attention-share exceeds $0.20$. $r_{\mathrm{cv,sh}}$ is the Pearson correlation between key-norm CV and max-attention-share across (layer,head) pairs ($\dagger$: $p \geq 0.05$). VE@8 values are means over layers.}
\label{tab:transformer-summary}
\end{table}

%% file: sections/05_crossarch.tex
\section{The signature is not specific to attention}
\label{sec:crossarch}

If the compensation signature is a property of content-based routing rather than
of softmax, it should appear in routers that are not transformers. We test four:
graph attention, a selective state-space model (Mamba), a recurrent model (RWKV),
and learned attention residuals (AttnRes). The first three route over tokens, nodes,
or time; the fourth routes over \emph{depth}, which tests whether the signature is
even specific to the token axis. Each scores and aggregates over a set of sources, so
each admits the same measurements, once we recover its routing weights. The signature
appears in all four. One-shot softmax attention reaches substantial concentration early
(9-25\% depth) and strongly (peak 0.57-0.78), while the two recurrent routers onset late
(42-81\%) and stay weak (peak 0.22-0.31), as Table~\ref{tab:crossarch} records and
Figure~\ref{fig:concentration-depth} shows as a depth profile: the transformer band
crosses early and plateaus, the recurrent routers stay flat and hump late. The mechanism
generalizes; the form does not.

\input{tables/tab_crossarch}

\begin{figure}[t]
\centering
\includegraphics[width=0.56\textwidth]{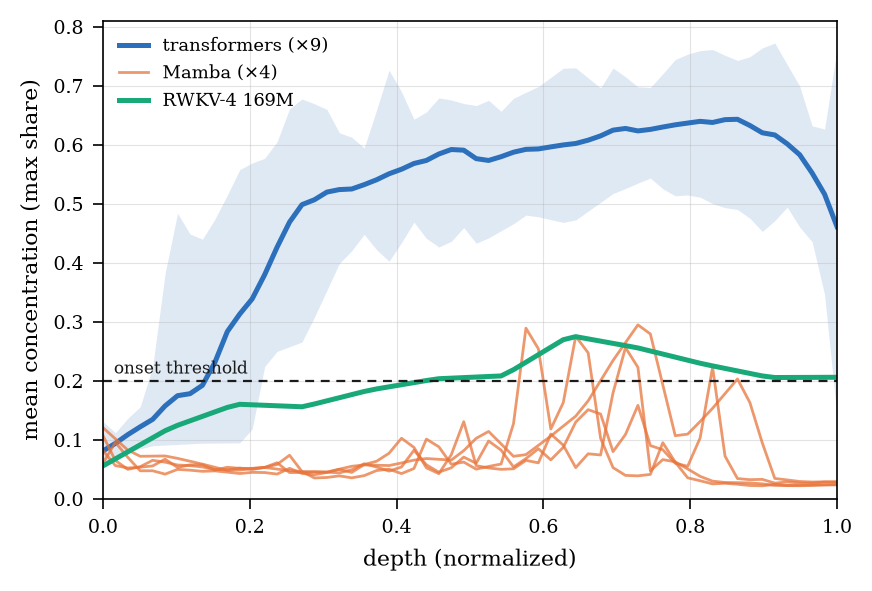}
\caption{Mean routing concentration against normalized depth. The nine transformers
(blue: min-max envelope and mean) cross the threshold near 15\% depth and plateau high;
the four Mamba scales (orange) hump late, 57-82\%; RWKV (green) onsets at 43\%. Transformer
values are post-softmax weights and the recurrent values are reconstructed operators, so read
onset depth and shape, not absolute height across the two.}
\label{fig:concentration-depth}
\end{figure}

For the two recurrent architectures we go further than measurement. Because their
routing weights are reconstructed from internal quantities we can hold fixed or
ablate, we can ask whether the routing mechanism \emph{causes} the signature, a
stronger question than whether it is present. In both cases it does.

\subsection{Graph attention, against an oversmoothing control}
Graph neural networks collapse representations with depth for a reason unrelated
to our hypothesis: oversmoothing, the tendency of repeated neighborhood
aggregation to wash node features toward a shared subspace
\citep{li2018deeper,oono2020graph}. To separate the two, we train a GAT and a
depth- and width-matched GCN that uses fixed symmetric-normalized aggregation
with no learned attention, and compare their rank-4 collapse on three WebKB
heterophilic graphs. The GAT collapses more than the
GCN on all three (VE@4 differential $+0.284$, $+0.183$, $+0.094$; full family in Table~\ref{tab:app-gat}), so the excess
collapse is specific to learned attention, not generic smoothing. We claim the
differential and nothing more: the accuracy comparison does not generalize, since
the GCN outperforms the GAT on Texas. The collapse is attention-specific; its
benefit is not.

\subsection{Mamba: concentration is caused by selectivity}
Mamba has no query-key attention. Its selective state-space layer can be unrolled
into a per-channel data-controlled operator \citep{ali2024hidden} whose causal
weights we reconstruct from the captured time-step, input, and output projections.
The concentration onsets late (56-81\% across four scales from 130M to 1.4B) and
stays weak (peak 0.22-0.31), and it is robust across scale without following a
clean scaling law.

The causal test is an ablation. Mamba's concentration is set by the input-dependent
gate $\Delta$, which controls how strongly each token writes to state. Freezing
$\Delta$ to its per-position mean, so the gate no longer selects, flattens the
concentration entirely: at every layer the real operator concentrates while the
frozen one does not, with a gap of $+0.23$ on average at the peak. Freezing the
input and output projections as well drives concentration to near uniform. So the
concentration is produced by the selectivity mechanism, not by generic recurrence.
The same ablation makes the norm story concrete: the effective key in this
operator is $\Delta$ times the input projection, so $\Delta$ is the per-position
key-norm control, and its coefficient of variation tracks the concentration peak
layer by layer while the static input-projection norm stays flat.

\subsection{RWKV: a second causal ablation, and a sharper temporal mechanism}
RWKV's time-mixing is a softmax over causal sources whose logit is a per-channel
time-decay times the source-query gap, plus the key. It is a recurrent router, and
unlike Mamba it is \emph{normalized}, so its co-late onset (42\% depth, peak 0.28)
also confirms that normalization is not what sets onset. The decay is a single
static parameter per channel, which makes it an explicit temporal knob: we can set
it and re-run the full model, so the change propagates across depth, and ask
whether the across-layer onset moves.

It moves, and the result is twofold. First, killing the carry, by forcing the
decay so steep that only the most recent source survives, removes the concentration
completely: it sits flat at $0.015$ at every layer and never onsets. Carry is
causally necessary for the concentration. Second, and against our own conjecture,
\emph{more} carry brings the onset \emph{earlier}, not later: with no decay at all,
concentration onsets at 17\% depth, versus 50\% for the learned decay. We had expected
temporal accumulation to distribute routing pressure and delay collapse; instead, carry
does not delay the attractor but creates it, and more of it creates more, sooner. The late onset of recurrent routers is real
(Table~\ref{tab:crossarch}) but is not caused by temporal distribution. What the
sweep does establish is the same thing the Mamba ablation establishes, from a
second architecture: the routing mechanism causally produces the concentration.
Figure~\ref{fig:causal-ablations}shows both ablations side by side.
\ifarxiv\input{sections/floats/fig3_causal_ablations}\fi

The three-point ablation is the coarse version of a continuous control. Scaling the
learned decay rate and re-running the model turns the three settings into a sweep
(Figure~\ref{fig:rwkv-sweep}), and two things follow. The concentration is a smooth,
monotone function of carry: peak concentration rises from $0.083$ under heavy decay
to $0.288$ at full carry, and the onset moves from never, through 50\% at the learned
decay, to 17\% at full carry. The three ablation points lie on this one curve.

The decomposition says what the decay is actually doing. At each carry level we score
every source by the routing it receives from the late queries, which all see it so
there is no exposure confound, and correlate that across sources with the source key
norm (content) and with its position (recency). The two trade off. At full carry the
winner is content-picked: influence correlates with key norm at $+0.36$ while the
recency correlation goes \emph{negative}, $-0.09$. At the learned decay the two are
balanced, $+0.17$ content against $+0.20$ position, and adding decay tips the balance
to position, whose correlation climbs toward $+0.43$. The decay is a positional brake
on a content-based attractor, and the trained model sits near the hinge. This is the
temporal mechanism the refuted conjecture was reaching for: not that memory delays
collapse, but that the decay sets how much the content metric expresses itself against
a recency prior. It places RWKV's onset on the same axis as
Theorem~\ref{thm:identity}, where concentration is what content routing does and the
positional term is the contingent brake, which Section~\ref{sec:hypothesis} takes up.

\subsection{AttnRes: the signature is not specific to the token axis}
The three routers above all route over tokens, nodes, or time. Attention residuals
\citep{kimi2026attnres} route over \emph{depth}: in place of the additive residual
$h_\ell = h_{\ell-1} + f(h_{\ell-1})$, each sublayer attends by softmax over the
previous block representations, scoring them with a learned query against
RMSNorm-normalized keys plus a recency bias on the current block. We measure an open
0.6B AttnRes model, a Qwen3 variant trained from scratch with zero-initialized queries,
so the routing starts uniform and any concentration is learned.

The depth routing concentrates. Averaged over its 56 sublayers the top source takes
$0.643$ of the routing weight against a uniform baseline of $0.245$, and the
concentration holds across depth (Figure~\ref{fig:attnres}, left). It piles onto two
hubs: the current block (recency, $0.42$) and the token embeddings (source 0, $0.26$).
The second is a depth-axis analog of the first-token attention sink, a fixed early
representation that downstream sublayers route back to.

This router sharpens one claim the others cannot. Its keys are RMSNorm-normalized, so
their norm coefficient of variation is zero by construction, and it concentrates anyway.
The norm stratification we measure in every transformer (Section~\ref{sec:transformers})
is therefore not necessary for the attractor; the concentration is a property of the
content direction, which normalization leaves untouched. Norm stratification is one
compensation a router can use, not the mechanism that makes it concentrate.

The recency bias gives a second causal handle. In this checkpoint the learned bias is
exactly zero, so the recency is itself content-driven, and we intervene by adding an
offset (Figure~\ref{fig:attnres}, right). Forcing the bias up routes everything to the
current block, the additive-residual limit, with the top share at $0.998$. Suppressing
it does not dissolve the concentration: with the current block driven to $0.001$ of the
weight, the top share holds at $0.677$, above its trained value, and the weight migrates
to the embedding hub, which rises to $0.339$. As in RWKV, the positional brake does not
create the attractor; removing it relocates the attractor from the recent source to a
content hub. An architecture built to route flexibly over depth, initialized to spread
uniformly, still collapses onto a few content-picked sources. The query here is a fixed
learned probe rather than a per-token query and the source set is small, so we read this
as the signature appearing on a new axis, not as a quantitative match to the token-axis
numbers.

\medskip
Across four architectures the signature is present, and in the two where we can
intervene it is caused by the routing mechanism rather than by incidental
dynamics. What differs is the form: onset depth, concentration strength, and the
particular subspace, the last of which a retraining control places in the appendix
(Appendix~\ref{app:results}). Section~\ref{sec:hypothesis} states the resulting
hypothesis and marks the line between what is invariant and what is not.

%% file: tables/tab_crossarch.tex
\begin{table}[t]
\centering
\small
\begin{tabular}{llrr}
\toprule
Architecture & routing & onset (depth) & peak conc. \\
\midrule
Transformers ($\times$9) & softmax, one-shot & 9-25\% & 0.57-0.78 \\
Mamba ($\times$4) & selective SSM, recurrent & 56-81\% & 0.23-0.31 \\
RWKV-4 169M & recurrent, normalized & 42\% & 0.28 \\
\bottomrule
\end{tabular}
\caption{Onset (first layer whose mean concentration exceeds 0.20) and peak concentration across architectures. One-shot softmax attention reaches substantial concentration early and strongly; the recurrent routers onset late and stay weak. (GAT omitted: its concentration is bounded by graph sparsity, not depth.) Transformer values use post-softmax weights; recurrent values use the reconstructed hidden-attention operator, so compare onset depths and relative strength, not absolute magnitudes across the two.}
\label{tab:crossarch}
\end{table}

%% file: sections/floats/fig3_causal_ablations.tex
\begin{figure}[htbp]
\centering
\includegraphics[width=\textwidth]{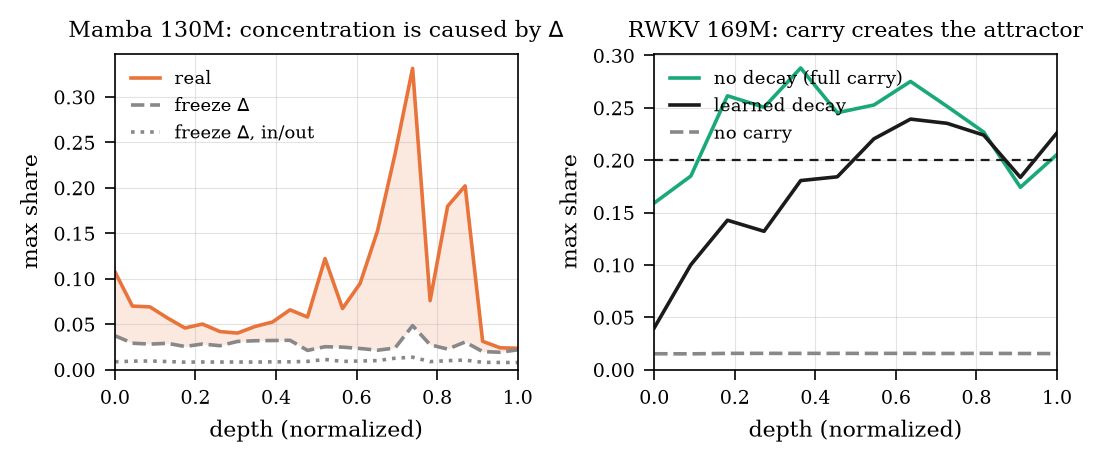}
\caption{The two within-model causal ablations. Left, Mamba: freezing the selective gate
$\Delta$ to its per-position mean (dashed) collapses the concentration the real operator
(solid) develops, and freezing the input and output projections too (dotted) drives it to
near uniform. Right, RWKV: killing the carry (dashed) holds concentration flat at $0.015$,
while maximal carry (green) brings the onset earlier than the learned decay (black). More
carry produces more concentration, sooner.}
\label{fig:causal-ablations}
\end{figure}

%% file: sections/06_hypothesis.tex
\section{The hypothesis: what is invariant and what is not}
\label{sec:hypothesis}

The evidence assembled so far points to a single mechanism with an
architecture-dependent form. We state it as a hypothesis.

\emph{Content-based routing under a fixed, norm-blind similarity metric develops a
representation attractor: the routing concentrates, the routed representations
collapse to a low-rank subspace, and their norms stratify. The mechanism is a
property of the routing, not of any one architecture or of the axis it routes over.}

The hypothesis separates an invariant from a contingent part. The invariant is the mechanism:
across the routers we tested (softmax attention over tokens, graph attention over nodes, a
selective state space and a recurrent mixer over time, and learned residuals over depth) the
same signature appears, and where we can intervene it is causally produced by the routing
rather than by incidental dynamics. The signature needs neither query-key attention (Mamba has
none), nor softmax normalization (Mamba is unnormalized, while normalized RWKV and AttnRes
concentrate anyway), nor the token axis (AttnRes routes over depth). The contingent part is the
form: the onset depth, the concentration strength, the particular subspace, and which source
becomes the hub, which vary across architectures and seeds and are set by architecture-specific
factors rather than by the mechanism.

\paragraph{The positional brake sets the form.}
The chief contingent factor is the strength of the positional term that every router
carries alongside its content score. Softmax attention pairs the content score with a
rotary phase, RWKV pairs it with a time-decay, and AttnRes pairs it with a recency
bias. Each is a brake that pulls routing toward recent or local sources and away from
the content-picked winner. Where the brake is weak the content attractor expresses
early and strongly; where it is strong the attractor expresses late and weakly, and we
can move the brake and watch the form move with it. Scaling RWKV's decay sweeps the
onset across the whole range, from 17\% depth at full carry to never under heavy decay,
and the late-query routing shifts from content-picked to recency-picked as the brake
tightens (Figure~\ref{fig:rwkv-sweep}). Suppressing AttnRes's recency bias does not
dissolve its concentration; it relocates it from the current block to content hubs. The
cross-architecture onset gap (transformers early, recurrent routers late) is therefore
the brake, not the mechanism: the recurrent routers carry an explicit decay, while
attention carries only the weak, non-monotone brake of its rotary phase.

\paragraph{Norm stratification is one compensation, not the mechanism.}
Section~\ref{sec:motivation} identifies the metric's blindness to key magnitude as the
inadequacy the routing must work around, and Section~\ref{sec:transformers} finds the
expected norm stratification in every transformer. But stratification is not necessary
for the attractor. AttnRes normalizes its keys, so their norm coefficient of variation
is zero by construction, and the routing concentrates anyway, onto content-picked hubs.
Norm stratification is one way a router can exploit a norm-blind metric, not the thing
that makes it concentrate; the concentration is a property of the content direction,
which normalization leaves untouched. This is why we read the geometry as a diagnosis of
the metric rather than as the disease itself.

\paragraph{What we do not claim.}
We do not claim these models implement Riemannian geometry: the geometric language of
Section~\ref{sec:motivation} is a diagnostic for the metric's inadequacy. Nor that the
attractor is always a pathology, since in the graph case a control without learned attention
sometimes generalizes better (Section~\ref{sec:crossarch}), so the collapse is specific to
learned attention but its cost is not universal. And we claim sufficiency, not necessity, since
oversmoothing reaches low rank by a different route. The hypothesis is falsifiable: a norm-blind
content router that did not develop the signature would refute the invariant claim, and a
demonstration that the concentration in our interventions comes from something other than the
routing, whether a token artifact, a normalization quirk, or a training-data regularity, would
refute the causal claim, which the Mamba, RWKV, GCN, and seed controls are built to close. The
brake account also predicts an ordering we have only partly tested, which
Section~\ref{sec:discussion} takes up.

%% file: sections/07_discussion.tex
\section{Discussion}
\label{sec:discussion}

\paragraph{A brake spectrum inside the transformer family.}
The account in Section~\ref{sec:hypothesis} makes an untested prediction the transformer family is
positioned to answer: every positional scheme is a brake of a different strength. With no positional
encoding there is no brake; RoPE \citep{su2021roformer} rotates queries and keys by position, a brake
whose distance decay is weak and non-monotone; ALiBi \citep{press2022alibi} adds a linear penalty
$-m\,|i-j|$, the same object as RWKV's decay times gap, an explicit and strong brake. Onset should
then order these from earliest to latest as no-encoding, RoPE, ALiBi: the weaker the brake, the
earlier and stronger the attractor. This is the transformer-internal version of the RWKV sweep in
Section~\ref{sec:crossarch}, directly testable on pretrained models.

\paragraph{The brake and the metric are separate axes.}
RoPE also clarifies what the positional term does not touch. A rotation preserves key norm, so it
leaves the metric's blindness to magnitude exactly where it was and changes only which source wins
by relative position. This is why we measure the Pythia key geometry before the rotary transform:
the pre-rotary keys are the content cloud the metric scores, and the rotation overlays position on
top. Across the five Pythia models it reduces the key variance explained, sharply at rank one,
where it roughly halves the single-direction dominance, and only slightly at rank eight
(Figure~\ref{fig:app-rope}): it spreads the geometry and resists the collapse without preventing
it: the post-rotary rank-eight VE stays far above the nulls. And since a rotation preserves norm,
the norm-CV is identical before and after, so RoPE cannot touch the blindness at all. The metric's
inadequacy and the positional brake are independent, and RoPE sits cleanly on the brake.

\paragraph{Convergent evidence.}
Three independent results corroborate the picture: attention residuals show the same depth-axis
sinks and a mechanism for the norm stratification, projection-sharing work finds the value
projection nearly redundant with the key and derives a collapse of linear attention into a
state-space recurrence, and the memory-caching view recovers attention as a cached recurrence.
Appendix~\ref{app:convergent} discusses these; none is load-bearing; our claim rests on the
measurements and the ablations.

\paragraph{Why low-rank adaptation may suffice.}
The collapse offers a routing-side reading of low-rank adaptation. If content routing has already
pressed the usable key geometry into a low-dimensional subspace, re-aiming that routing should need
only a low-rank correction, which is what LoRA supplies \citep{hu2022lora}, alongside the low
intrinsic dimensionality already observed for fine-tuning \citep{aghajanyan2021intrinsic}, now with
a reason on the representation side rather than the optimization side. We do not test it: our
low-rankness is in the forward-pass key geometry while adaptation acts in weight space, so this is
a connection, not a result.

\paragraph{Limitations.}
The models are small to mid-scale, and the signature is robust across the scales we test without
following a clean scaling law. The AttnRes query is a fixed learned probe and its source set is
small, so the depth-axis result establishes the signature on a new axis but not a quantitative
match to the token-axis numbers. The causal interventions reach out-of-distribution settings at
their extremes by design, so they bound what a parameter controls, not the regime the trained
model occupies. And the geometric framing is a diagnosis, not an implementation claim.

%% file: sections/08_conclusion.tex
\section{Conclusion}
\label{sec:conclusion}

Attention sinks, representation collapse, and norm stratification are not specific to
attention. We gave a reframing identity exposing the flat, norm-blind metric in softmax
routing, and predicted that any router sharing it should develop the same signature. Across
four architecturally distinct routers over tokens, time, and depth, and an oversmoothing
control on graphs, the signature appears, and where we can intervene it is caused by the
routing, not incidental dynamics. What varies is the form: onset depth, strength, and
subspace, set mainly by the positional brake each router carries. What does not vary is the
mechanism, and the norm story is one compensation among several, not the cause: a
norm-normalized router concentrates just the same. The routes differ; the destination does not.

\ifarxiv
Three handles remain for future work. The brake account predicts an onset ordering across
position encodings, from none through RoPE to ALiBi, testable on existing models. The expert
axis of mixture-of-experts routing \citep{shazeer2017moe} is a fifth place the mechanism should
appear, where the load-balancing losses already in use act as an imposed brake against it. And
whether the collapse should be corrected or exploited is left open: in one of our settings a
router without learned attention generalizes better, so the attractor is specific to learned
routing while its cost is not universal.
\fi

%% file: sections/appendix_theory.tex
\section{The reframing identity: assumptions and proofs}
\label{app:theory}

This appendix states the four assumptions behind Theorem~\ref{thm:identity}, derives the
Boltzmann weighting from a maximum-entropy principle, proves the identity, and proves the claim of
Section~\ref{sec:motivation} that the metric correction lies outside the bilinear, query-key
family.

\subsection{The four assumptions}
The identity reads standard attention as Boltzmann routing over a flat, norm-uniform key geometry,
and it holds under the four assumptions of Table~\ref{tab:assumptions}. A1 through A3 are a
geometric lens: they make the metric explicit and flat, and we use them to name the metric the
router commits to, not to claim that a trained network performs geometry on a manifold. A4 is the
assumption with empirical content, and Section~\ref{sec:transformers} finds it violated in every
layer of every model, which is the precise sense in which the dot-product score is blind to key
magnitude.

\begin{table}[htbp]
\centering
\caption{The four assumptions behind the reframing identity (Theorem~\ref{thm:identity}). A1
through A3 are the geometric lens, the conditions under which the geodesic distance is Euclidean up
to a constant; A4 is the only assumption with empirical content, and it is violated.}
\label{tab:assumptions}
\begin{tabular}{@{}clll@{}}
\toprule
 & Assumption & Geometric meaning & Status \\
\midrule
A1 & smooth manifold & keys carry a local geometry & lens \\
A2 & conformally flat metric & $g = e^{2\phi}I$, no directional bias & lens \\
A3 & constant conformal factor & $g = \Omega^2 I$, straight geodesics & lens \\
A4 & hyperspherical keys & $\lVert k_i\rVert = c$, keys on a sphere & \emph{empirical, violated} \\
\bottomrule
\end{tabular}
\end{table}

\subsection{Boltzmann weighting from maximum entropy}
The Boltzmann form is derived, not posited. A distance-based router that commits to the least
biased weighting at a fixed expected distance maximizes entropy subject to that constraint, and the
solution is unique.

\begin{proposition}
\label{prop:maxent}
Among all weightings $\alpha$ on the simplex $\Delta^{n-1}$, the maximizer of the Shannon entropy
$H(\alpha) = -\sum_i \alpha_i \log \alpha_i$ subject to a fixed expected squared distance
$\sum_i \alpha_i D(q,k_i)^2 = \bar{D}$ is the Boltzmann distribution
$\alpha_i \propto \exp(-\beta\, D(q,k_i)^2)$, with $\beta$ the multiplier of the constraint.
\end{proposition}

\begin{proof}
The Lagrangian is
\[
\mathcal{L} = -\sum_i \alpha_i \log \alpha_i
- \beta\Big(\textstyle\sum_i \alpha_i D(q,k_i)^2 - \bar{D}\Big)
- \lambda\Big(\textstyle\sum_i \alpha_i - 1\Big).
\]
Stationarity, $\partial\mathcal{L}/\partial\alpha_i = -\log\alpha_i - 1 - \beta\, D(q,k_i)^2 - \lambda = 0$,
gives $\alpha_i \propto \exp(-\beta\, D(q,k_i)^2)$; the normalization constraint fixes the constant.
The entropy is strictly concave and the constraints are linear, so this stationary point is the
unique global maximum.
\end{proof}

\subsection{Proof of the identity}
\begin{proof}[Proof of Theorem~\ref{thm:identity}]
By A1 through A3 the metric is flat with a constant conformal factor, so the geodesic distance is
Euclidean up to a constant, $D(q,k_i)^2 = \Omega^2\lVert q-k_i\rVert^2$. Take the Boltzmann weight
of Proposition~\ref{prop:maxent} with this distance and expand the squared norm,
\[
\lVert q-k_i\rVert^2 = \lVert q\rVert^2 - 2\inner{q}{k_i} + \lVert k_i\rVert^2 .
\]
Substituting into the numerator and denominator of \eqref{eq:boltzmann},
\[
\alpha_i = \frac{\exp\!\big(-\beta\Omega^2\lVert q\rVert^2\big)\,
\exp\!\big(-\beta\Omega^2\lVert k_i\rVert^2\big)\,
\exp\!\big(2\beta\Omega^2\inner{q}{k_i}\big)}
{\sum_j \exp\!\big(-\beta\Omega^2\lVert q\rVert^2\big)\,
\exp\!\big(-\beta\Omega^2\lVert k_j\rVert^2\big)\,
\exp\!\big(2\beta\Omega^2\inner{q}{k_j}\big)} .
\]
The factor $\exp(-\beta\Omega^2\lVert q\rVert^2)$ is independent of the index and cancels. By A4,
$\lVert k_i\rVert = c$ for every $i$, so $\exp(-\beta\Omega^2\lVert k_i\rVert^2) = \exp(-\beta\Omega^2 c^2)$
is also constant in $i$ and cancels. What survives is
\[
\alpha_i = \frac{\exp\!\big(2\beta\Omega^2\inner{q}{k_i}\big)}
{\sum_j \exp\!\big(2\beta\Omega^2\inner{q}{k_j}\big)}
= \softmax_i\!\big(2\beta\Omega^2\inner{q}{k_i}\big).
\]
Setting $2\beta\Omega^2 = 1/\sqrt{d}$ gives $\alpha_i = \softmax_i(\inner{q}{k_i}/\sqrt{d})$, the
scaled dot-product attention weights.
\end{proof}

\subsection{The metric correction is outside the bilinear family}
The diagnosis of Section~\ref{sec:motivation} restores the dropped $\lVert k\rVert^2$ term, giving a
learned-metric score $s_M(q,k) = -(q-k)^\top M (q-k)$. The following makes precise why this is not
a reparameterization of the existing query and key projections.

\begin{proposition}
\label{prop:bilinear}
Let $M \neq 0$ be symmetric. The score $s_M(q,k) = -(q-k)^\top M (q-k)$ cannot be written as a
bilinear form $q^\top A k$ for any matrix $A$. Since standard attention scores bilinearly, with
$q^\top A k$ and $A = W_Q^\top W_K$, no choice of query and key projections reproduces $s_M$.
\end{proposition}

\begin{proof}
Expand $s_M(q,k) = -q^\top M q + 2\,q^\top M k - k^\top M k$. Only the cross term $2\,q^\top M k$ is
bilinear; the remaining two are a key-independent term $-q^\top M q$ and a query-independent term
$-k^\top M k$. Suppose for contradiction that $s_M(q,k) = q^\top A k$ for some $A$ and all $q,k$.
Setting $q=0$ gives $-k^\top M k$ on the left and $0$ on the right, so $k^\top M k = 0$ for every
$k$, which forces $M = 0$ for symmetric $M$, contradicting $M \neq 0$. Hence no bilinear score, and
therefore no reparameterization of $W_Q, W_K$, reproduces $s_M$: the correction is a new degree of
freedom rather than a re-weighting of the existing one.
\end{proof}

%% file: sections/appendix_methods.tex
\section{Reconstruction and measurement details}
\label{app:methods}

\subsection{Models and data}
The nine transformers (GPT-2 small/medium/large/XL and Pythia 160M/410M/1B/1.4B/2.8B) are
the pretrained HuggingFace checkpoints. We run 150 sequences of length 128 from the
WikiText-103 validation split in fp32 and read the keys with forward hooks, taking them
before and after the rotary transform for the Pythia models. The graph attention networks
are trained from scratch on three WebKB heterophilic graphs (Cornell, Texas, Wisconsin),
6 layers, hidden width 64, 8 heads, alongside a depth- and width-matched GCN control. The
Mamba models (130M/370M/790M/1.4B) are run through HuggingFace's reference \texttt{slow\_forward}
so that hooks on the time-step and input/output projections expose the per-step internals; RWKV
uses \texttt{rwkv-4-169m-pile}; AttnRes uses the open \texttt{0.6B} block checkpoint (a Qwen3
variant, 28 layers, 8 blocks). For the three reconstructed-operator models we use 16 sequences
of length 128, since the per-channel reconstruction is heavier than reading attention directly.

\subsection{Recovering routing weights}
For softmax attention and graph attention the routing weights are the attention coefficients,
read straight from the forward pass; for GAT the key is the projected source-node feature and
the weight is the softmax over a node's incoming edges, with the GCN control replacing the
learned attention by fixed symmetric-normalized aggregation. The three remaining routers have
no explicit attention, so we reconstruct an attention-like operator from internal quantities.

\paragraph{Mamba.}
The selective state-space recurrence $h_t = \bar{A}_t h_{t-1} + \bar{B}_t x_t$, $y_t = C_t h_t$,
with input-dependent $\bar{A}_t = \exp(\Delta_t A)$ and $\bar{B}_t = \Delta_t B_t$, unrolls into a
per-channel data-controlled operator \citep{ali2024hidden},
\[
y_t = \sum_{s \le t} \alpha_{t,s}\, x_s, \qquad
\alpha_{t,s} = C_t \Big(\textstyle\prod_{r=s+1}^{t} \bar{A}_r\Big) \bar{B}_s .
\]
We capture $\Delta_t, B_t, C_t$ from the time-step and input/output projections and form
$\alpha_{t,s}$. The effective key of source $s$ is $\bar{B}_s = \Delta_s B_s$, so its norm is
$\Delta_s \lVert B_s \rVert$ and $\Delta$ is the per-position key-norm control.

\paragraph{RWKV.}
The WKV time-mixing is a softmax over causal sources. For channel $c$, query $t$, and source
$s < t$ the weight is $\alpha_{t,s} \propto \exp\!\big((t-1-s)\,w_c + k_{s,c}\big)$, with the
current token weighted by $\exp(u_c + k_{t,c})$ and the distribution normalized over $s \le t$.
Here $w_c = -\exp(\texttt{time\_decay}_c) < 0$ is the per-channel decay, $u_c$ is the
current-token bonus, and $k$ is the key projection, all read from the forward pass.

\paragraph{AttnRes.}
Each sublayer routes over the running block representations $V = \{s_1,\dots,s_N,\text{partial}\}$.
With $K = \mathrm{RMSNorm}(V)$ and a single learned per-sublayer query vector $q$, the logits are
$q^\top K$, a recency bias is added to the current block, and the weights are a softmax over the
$N{+}1$ source blocks. Because the keys are norm-normalized, their norm coefficient of variation
is zero by construction. We extract the weights with the repository's hook.

\subsection{Measurements}
\paragraph{Concentration.} For each operator the column sum gives the total weight a source
receives across queries; the max share is the largest source's normalized column sum. Onset is
the first layer whose mean max share exceeds $0.20$ (Section~\ref{sec:transformers}).
\paragraph{Low-rank collapse.} We form the matrix of squared pairwise distances between keys and
double-center it. By the classical-MDS identity the double-centered squared-distance matrix
equals $-2 X_c X_c^\top$ for centered keys $X_c$, so its nonzero spectrum is that of the centered
key Gram matrix; VE@$r$ is the rank-$r$ variance explained from its singular values, averaged
over depth.
\paragraph{Norm stratification.} The norm-CV is the coefficient of variation of the key norms,
compared against the isotropic-Gaussian baseline $1/\sqrt{2 d_h}$.

\subsection{Null baselines}
The Gaussian null draws i.i.d.\ Gaussian keys with the observed mean and variance; the centered
Gram matrix is then a Wishart matrix, whose limiting spectrum, and therefore VE@$r$, follows the
Marchenko-Pastur law and depends only on the aspect ratio $c = n/d_h$ for $n=128$ tokens. The
shuffle null permutes token positions within each sequence, preserving the key marginal
distribution while destroying token-order structure. (For graphs the shuffle null is degenerate,
since permuting which node owns which key leaves the distance matrix unchanged, so we use the GCN
control instead.)

\subsection{Causal interventions}
Three within-model interventions test whether the routing mechanism produces the concentration.
In Mamba we freeze the selective gate $\Delta$ to its per-position mean so it no longer selects,
and optionally freeze the input and output projections as well. In RWKV we scale, or additively
offset, the learned decay and re-run the full model so the change propagates across depth. In
AttnRes we add an offset to the (zero-valued) learned recency bias and re-extract the weights.

%% file: sections/appendix_data.tex
\section{Supplementary results}
\label{app:results}

This appendix collects the per-architecture measurements that the main text summarizes.

\subsection{Low-rank collapse against matched nulls}
Figure~\ref{fig:app-nulls} shows the low-rank collapse of the routing keys clearing a matched
null in every architecture that supplies one, the recurrent routers as well as the transformers.

\begin{figure}[htbp]
\centering
\includegraphics[width=0.95\textwidth]{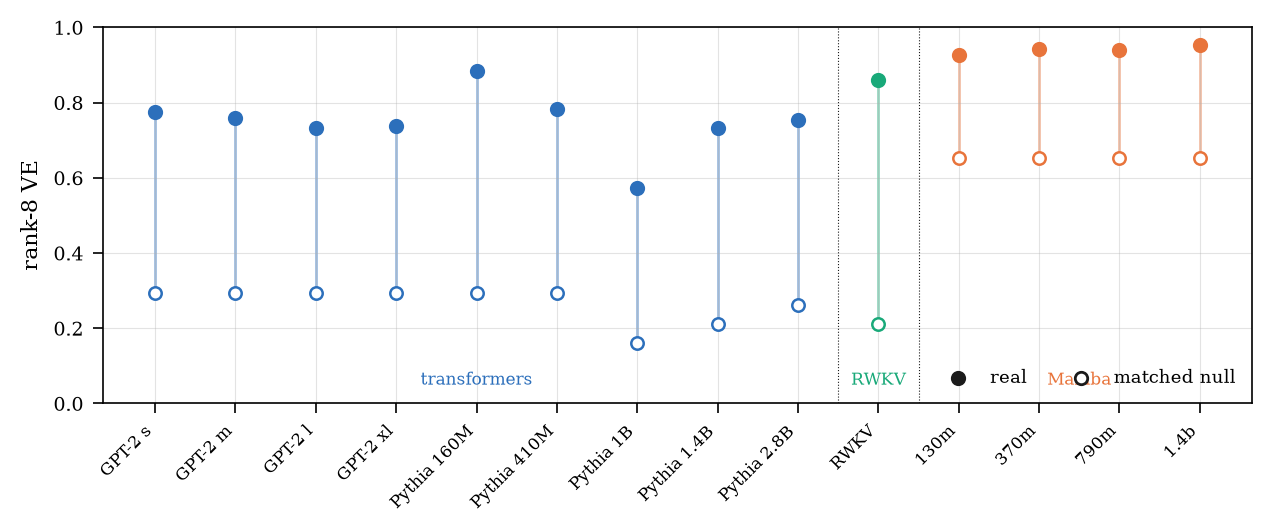}
\caption{Rank-8 variance explained of the routing key geometry (filled) against its matched
null (open), for every architecture that supplies a null: the nine transformers (Gaussian, that
is Marchenko-Pastur, null on the key distance matrix), RWKV (null on the key geometry), and the
four Mamba scales (null on the input matrix $B$). The real value clears the null in all fourteen
cases. The null level differs across architectures because it depends on the dimensionality of
the routed object, so the comparison is within-architecture; GAT is omitted here because its
baseline is the GCN control of Table~\ref{tab:app-gat}, not a Gaussian null.}
\label{fig:app-nulls}
\end{figure}

Figure~\ref{fig:collapse-nulls} resolves the same comparison as a depth profile for three
representative transformers, with the content and marginal effects of
Table~\ref{tab:app-transformer-full} shown as shaded bands.

\begin{figure}[htbp]
\centering
\includegraphics[width=\textwidth]{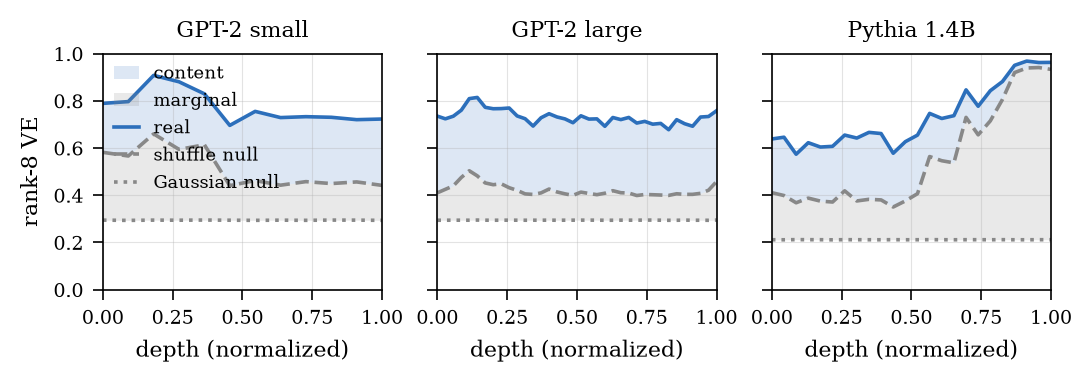}
\caption{Rank-8 variance explained of the double-centered key distance matrix across depth, for
three representative transformers, against the shuffle and Gaussian (Marchenko-Pastur) nulls. Real
keys (solid) clear both nulls at every depth. The shaded bands are the decomposition of
Table~\ref{tab:app-transformer-full}: the content effect (real minus shuffle) above, the marginal
effect (shuffle minus Gauss) below. The content band widens with scale, widest for Pythia~1.4B, and
the Gaussian null sits lower for Pythia~1.4B because its larger head dimension lowers the
Marchenko-Pastur value.}
\label{fig:collapse-nulls}
\end{figure}

\subsection{Full transformer measurements}
Table~\ref{tab:app-transformer-full} gives every per-model transformer metric in one place,
consolidating the two main-text tables and adding the per-layer minimum norm-CV and the
completion depth $L^\ast$.

\input{tables/tab_app_transformer_full}

\subsection{RoPE shifts the geometry, not the norm}
Figure~\ref{fig:app-rope} measures the Pythia key geometry before and after the rotary transform,
the evidence for the separation of brake and metric in Section~\ref{sec:discussion}: the rotation
spreads the key directions, lowering their variance explained, while leaving the key norms untouched.

\begin{figure}[htbp]
\centering
\includegraphics[width=0.72\textwidth]{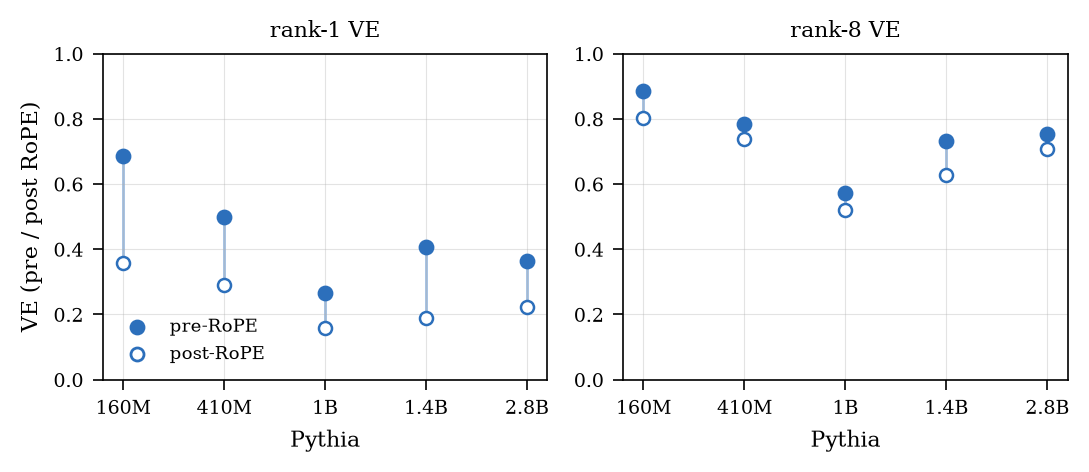}
\caption{Key variance explained before (filled) and after (open) the rotary transform, for the
five Pythia models, at rank 1 and rank 8. The rotation reduces VE everywhere, strongly at rank 1,
where it roughly halves the single-direction dominance, and weakly at rank 8, so it spreads the
keys angularly and resists the collapse without preventing it: the post-rotary rank-8 VE stays far
above the nulls of Figure~\ref{fig:app-nulls}. A rotation preserves norm, so the norm-CV is
identical before and after and is not plotted.}
\label{fig:app-rope}
\end{figure}

\subsection{The RWKV decay sweep}
Figure~\ref{fig:rwkv-sweep} reports the peak concentration and the content/position correlation
across the sweep; Figure~\ref{fig:app-rwkv-detail} adds the full depth profile at every decay scale,
where as carry rises the whole curve lifts and its onset moves earlier, from never crossing the
threshold under heavy decay to crossing near 17\% depth at full carry.

\begin{figure}[htbp]
\centering
\includegraphics[width=0.85\textwidth]{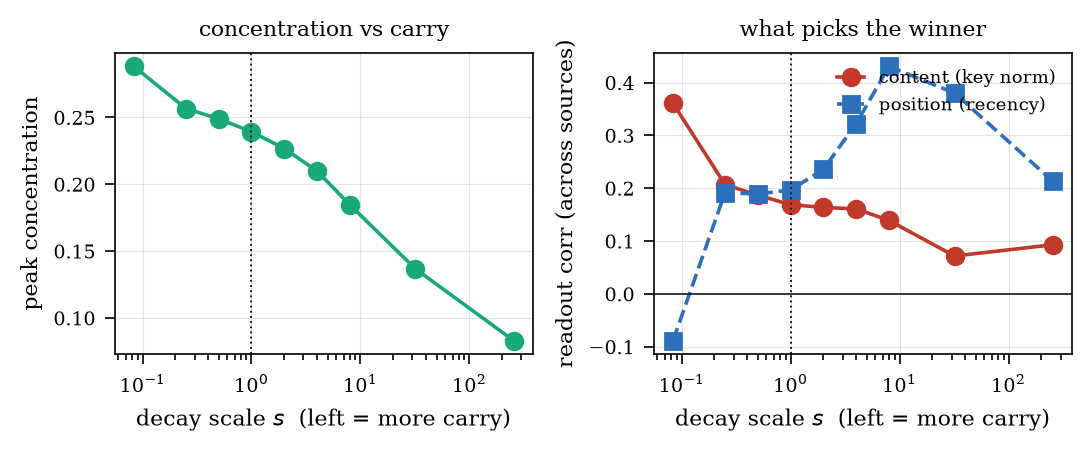}
\caption{Scaling RWKV's learned decay turns the three-point ablation into a continuous sweep
($s=1$ is the learned model, smaller $s$ is more carry). Left: peak concentration is a smooth
monotone function of carry, from $0.08$ under heavy decay to $0.29$ at full carry. Right: at each
carry level, the late-query routing correlated across sources with key norm (content) and with
position (recency). The decay trades the two: at full carry the winner is content-picked (content
$+0.36$, recency $-0.09$), at the learned decay they are balanced, and adding decay tips it to
position. The decay is a positional brake on a content-based attractor, and the trained model sits
near the hinge.}
\label{fig:rwkv-sweep}
\end{figure}

\begin{figure}[htbp]
\centering
\includegraphics[width=0.72\textwidth]{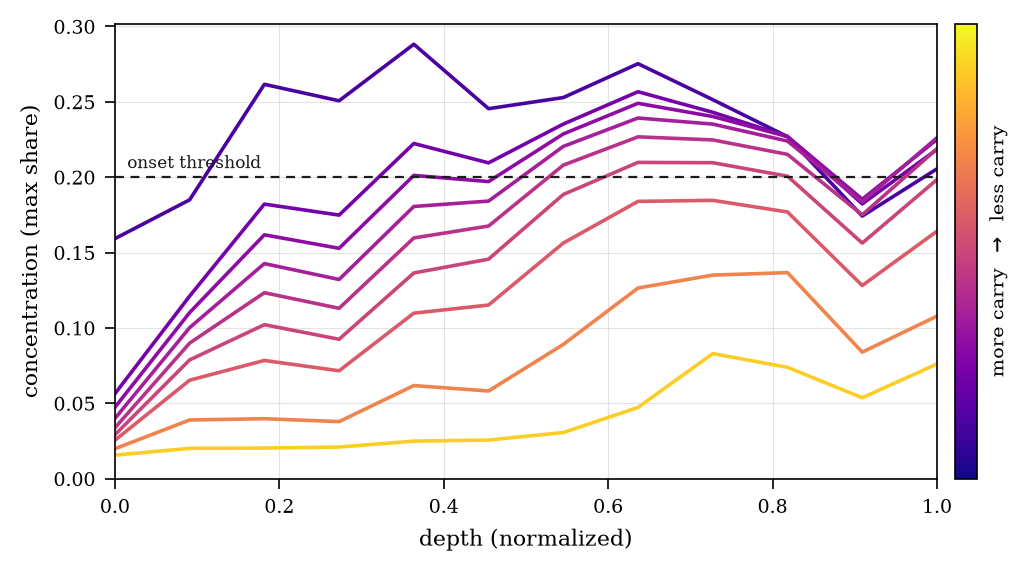}
\caption{RWKV concentration against depth at each of the nine decay scales of
Section~\ref{sec:crossarch} (color: dark is more carry, light is less). The dashed line is the
$0.20$ onset threshold. Each curve is one setting of the swept decay; the family shows the onset
sweeping earlier and the whole profile lifting as carry increases, which the peak-only summary in
Figure~\ref{fig:rwkv-sweep} cannot show.}
\label{fig:app-rwkv-detail}
\end{figure}

\subsection{Per-architecture detail}
Tables~\ref{tab:app-mamba}, \ref{tab:app-gat}, and \ref{tab:app-attnres} give the measurements
behind the cross-architecture summary: the Mamba scale ladder, the GAT-versus-GCN graph family,
and the AttnRes recency-offset sweep.

\input{tables/tab_app_mamba}
\input{tables/tab_app_gat}
\input{tables/tab_app_attnres}

Figure~\ref{fig:app-attnres-profile} resolves the AttnRes hubs by depth: the token embeddings
dominate the early sublayers (the depth sink), the current block dominates the middle, and the
late sublayers route from specific intermediate blocks where recency dips but concentration stays.

\begin{figure}[htbp]
\centering
\includegraphics[width=0.72\textwidth]{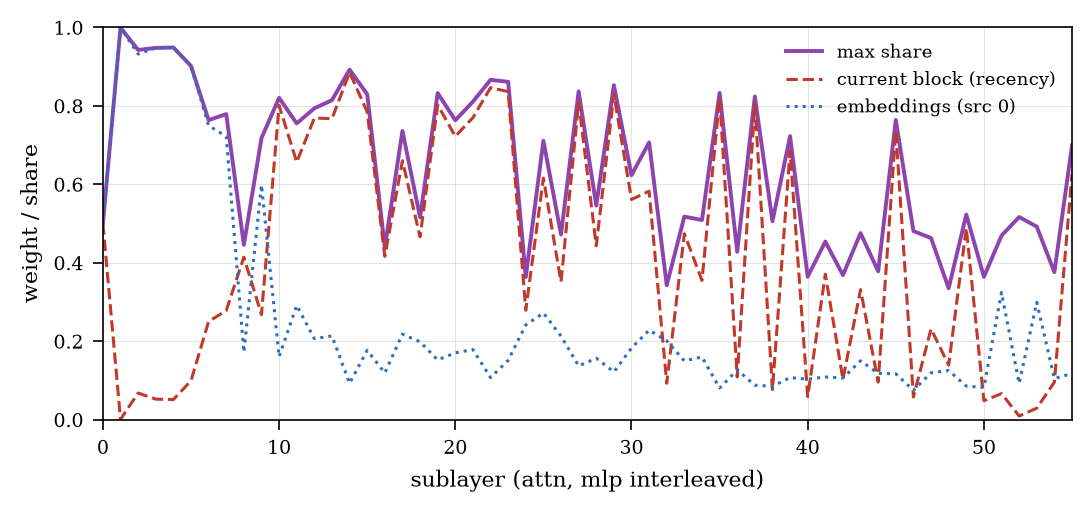}
\caption{AttnRes per-sublayer hub composition: the share taken by the top source (max share),
the current block (recency), and the token embeddings (source 0), across the 56 sublayers. Early
sublayers route from the embeddings, the middle from the current block, and the late ones from
intermediate content blocks.}
\label{fig:app-attnres-profile}
\end{figure}

\begin{figure}[htbp]
\centering
\includegraphics[width=0.85\textwidth]{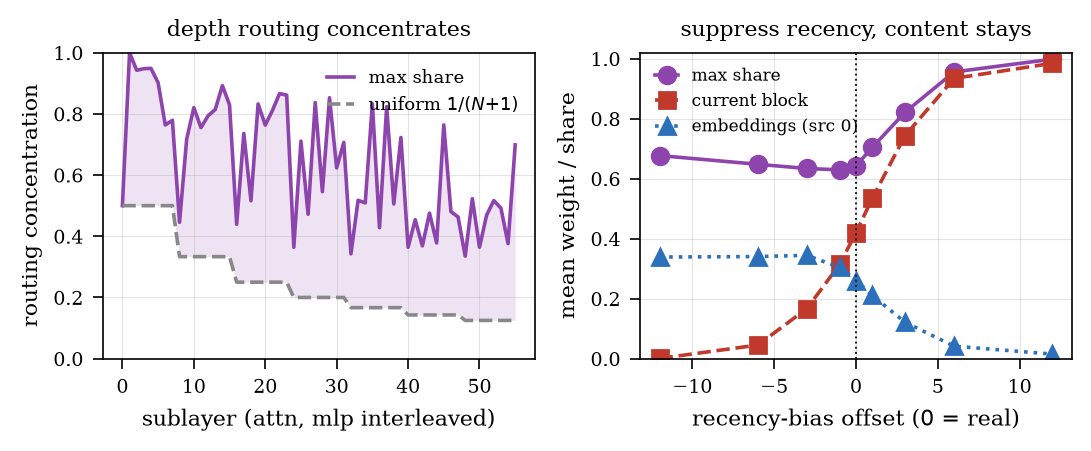}
\caption{AttnRes routes over depth, and the signature appears there too. Left: across the 56
sublayers the top source's share of the routing (max share) sits well above the uniform baseline
$1/(N{+}1)$, so the routing concentrates throughout depth. Right: the learned recency bias is zero,
so we offset it. Forcing it ($+$) routes all weight to the current block (the additive-residual
limit); suppressing it ($-$) drives the current block to zero yet leaves the concentration intact,
with the weight migrating to the token-embedding hub. The concentration is content-driven, not
positional.}
\label{fig:attnres}
\end{figure}

\subsection{Direction versus geometry under retraining}
A final control isolates what is contingent in the attractor from what is structural. Training two
copies of a 12-layer Pythia from different random seeds, we compare their attractor subspaces by
canonical angle and their attractor geometry by VE. The geometry matches across seeds (rank-8 VE
differs by under $0.1$ throughout) while the direction does not (canonical angles well above zero
past the transition). Both weight-initialization and data-order seeds break the direction. The
attractor that forms is structurally determined; the particular subspace it forms in is an accident
of training (Figure~\ref{fig:direction-geometry}).

\begin{figure}[htbp]
\centering
\includegraphics[width=0.85\textwidth]{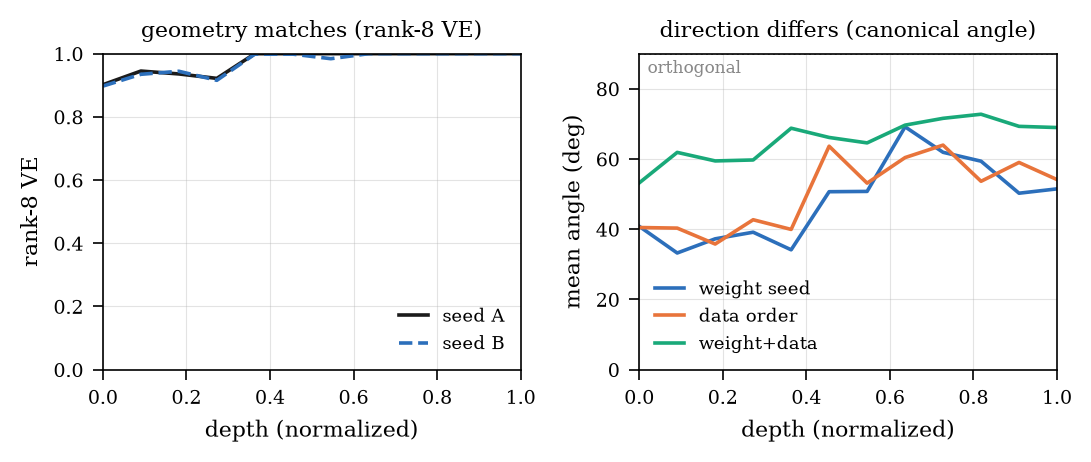}
\caption{Direction versus geometry under retraining. Left: rank-8 VE for two 12-layer Pythia models
trained from different weight seeds is nearly identical at every depth (maximum difference $0.016$),
so the attractor geometry is reproducible. Right: the mean canonical angle between their attractor
subspaces is large at every depth and for all three reseed types (weight, data order, and both),
well away from zero and toward the orthogonal ceiling of 90 degrees, so the particular subspace is
not reproducible. The geometry is structural; the direction is an accident of training.}
\label{fig:direction-geometry}
\end{figure}

%% file: tables/tab_app_transformer_full.tex
\begin{table}[htbp]
\centering
\scriptsize
\setlength{\tabcolsep}{4pt}
\begin{tabular}{lrrrrrrrrrrrr}
\toprule
Model & $L$ & $d_h$ & norm-CV & min & VE@8 & shuf & Gauss & cont. & marg. & onset & $L^\ast$ & $r$ \\
\midrule
GPT-2 small & 12 & 64 & 0.151 & 0.126 & 0.774 & 0.514 & 0.294 & 0.260 & 0.220 & 17\% & 33\% & 0.37 \\
GPT-2 medium & 24 & 64 & 0.210 & 0.114 & 0.758 & 0.471 & 0.294 & 0.287 & 0.176 & 17\% & 21\% & 0.45 \\
GPT-2 large & 36 & 64 & 0.249 & 0.159 & 0.732 & 0.422 & 0.294 & 0.310 & 0.127 & 22\% & 31\% & 0.59 \\
GPT-2 XL & 48 & 64 & 0.307 & 0.143 & 0.737 & 0.415 & 0.294 & 0.322 & 0.120 & 17\% & 60\% & 0.70 \\
\midrule
Pythia 160M & 12 & 64 & 0.258 & 0.157 & 0.885 & 0.819 & 0.294 & 0.066 & 0.525 & 17\% & 33\% & 0.28 \\
Pythia 410M & 24 & 64 & 0.260 & 0.185 & 0.783 & 0.636 & 0.294 & 0.147 & 0.342 & 25\% & 25\% & 0.05 \\
Pythia 1B & 16 & 256 & 0.254 & 0.157 & 0.572 & 0.370 & 0.160 & 0.202 & 0.210 & 25\% & 25\% & 0.18 \\
Pythia 1.4B & 24 & 128 & 0.289 & 0.197 & 0.733 & 0.554 & 0.212 & 0.179 & 0.342 & 12\% & 17\% & 0.19 \\
Pythia 2.8B & 32 & 80 & 0.286 & 0.227 & 0.753 & 0.539 & 0.262 & 0.214 & 0.277 & 9\% & 19\% & 0.07 \\
\bottomrule
\end{tabular}
\caption{Full per-model transformer measurements (WikiText-103, $N{=}150$, sequence length 128). norm-CV (mean) and its per-layer minimum are the key-norm coefficient of variation; VE@8 is the depth-averaged rank-8 variance explained of the key distance matrix, with its shuffle and Gaussian (Marchenko-Pastur) nulls; cont.\ and marg.\ are the content (real minus shuffle) and marginal (shuffle minus Gauss) parts; onset and $L^\ast$ are the first layers whose mean concentration exceeds $0.20$ and $0.50$; $r$ is the key-norm/concentration correlation. The Gaussian null depends only on $d_h$, the Marchenko-Pastur ladder: $0.294$ at $d_h{=}64$, $0.262$ at $80$, $0.212$ at $128$, $0.160$ at $256$.}
\label{tab:app-transformer-full}
\end{table}

%% file: tables/tab_app_mamba.tex
\begin{table}[htbp]
\centering
\small
\begin{tabular}{lrrrrrr}
\toprule
Mamba & $L$ & onset & peak & $\Delta$-gap & VE$_B$@8 & null \\
\midrule
130M & 24 & 67\% & 0.31 & $+0.28$ & 0.93 & 0.65 \\
370M & 48 & 81\% & 0.23 & $+0.17$ & 0.94 & 0.65 \\
790M & 48 & 69\% & 0.28 & $+0.24$ & 0.94 & 0.65 \\
1.4B & 48 & 56\% & 0.29 & $+0.22$ & 0.95 & 0.65 \\
\bottomrule
\end{tabular}
\caption{Mamba selectivity across four scales. onset is the first layer whose mean hidden-attention concentration exceeds 0.20; peak is its maximum over depth; $\Delta$-gap is the drop in peak concentration when the selective gate $\Delta$ is frozen to its per-position mean, the causal effect of selectivity; VE$_B$@8 and its null are the rank-8 variance explained of the input matrix $B$ and its matched Gaussian null. The signature is robust across scale without following a clean scaling law.}
\label{tab:app-mamba}
\end{table}

%% file: tables/tab_app_gat.tex
\begin{table}[htbp]
\centering
\small
\begin{tabular}{lrrrrrr}
\toprule
Graph & nodes & GAT acc & GCN acc & VE@4 GAT & VE@4 GCN & diff \\
\midrule
Cornell & 183 & 0.622 & 0.405 & 0.969 & 0.686 & $+0.284$ \\
Texas & 183 & 0.568 & 0.649 & 0.956 & 0.773 & $+0.183$ \\
Wisconsin & 251 & 0.667 & 0.549 & 0.942 & 0.848 & $+0.094$ \\
\bottomrule
\end{tabular}
\caption{Graph attention versus a depth- and width-matched GCN control on three WebKB graphs. VE@4 is the depth-averaged rank-4 variance explained of the node features; the differential (GAT minus GCN) is the attention-specific collapse on top of generic oversmoothing, positive on all three. The accuracy comparison does not generalize (GCN wins on Texas), so we claim the collapse differential, not an accuracy coupling.}
\label{tab:app-gat}
\end{table}

%% file: tables/tab_app_attnres.tex
\begin{table}[htbp]
\centering
\small
\begin{tabular}{rrrrr}
\toprule
offset $\delta$ & max share & norm.\ entropy & frac.\ recency & frac.\ src-0 \\
\midrule
$-12$ & 0.677 & 0.471 & 0.001 & 0.339 \\
$-6$ & 0.648 & 0.530 & 0.045 & 0.340 \\
$-3$ & 0.634 & 0.570 & 0.167 & 0.344 \\
$-1$ & 0.629 & 0.588 & 0.314 & 0.304 \\
$+0$$^{\star}$ & 0.643 & 0.563 & 0.420 & 0.260 \\
$+1$ & 0.705 & 0.505 & 0.535 & 0.211 \\
$+3$ & 0.821 & 0.338 & 0.744 & 0.121 \\
$+6$ & 0.956 & 0.110 & 0.934 & 0.042 \\
$+12$ & 0.998 & 0.006 & 0.985 & 0.015 \\
\bottomrule
\end{tabular}
\caption{AttnRes depth routing under an additive offset to the (zero-valued) learned recency bias; $\star$ marks the trained model. Forcing the bias up ($\delta>0$) routes all weight to the current block, the additive-residual limit; suppressing it ($\delta<0$) drives the current block toward zero yet keeps the concentration intact (max share stays near its trained $0.64$), with the weight migrating to the token-embedding hub. The concentration is content-driven, not positional.}
\label{tab:app-attnres}
\end{table}

%% file: sections/appendix_convergent.tex
\section{Convergent evidence from independent work}
\label{app:convergent}

Several recent results, from groups with unrelated aims, arrive at pieces of the same picture. We
collect them here. None is load-bearing for the claims of the main text, which rest on the
measurements of Sections~\ref{sec:transformers} and~\ref{sec:crossarch} and the within-model
ablations; what follows is corroboration, reached by different routes.

\paragraph{Attention residuals: depth-axis sinks and a norm mechanism.}
\citet{kimi2026attnres} replace the additive residual with learned softmax attention over previous
layers, and report that a layer attends selectively to a few specific earlier layers rather than
the most recent one, with persistent weight on the token embeddings and a recency bias. That is the
depth-axis attractor we measure in Section~\ref{sec:crossarch}, read by them as a useful routing
pattern; their learned weights show the same early-embedding, mid-recency, late-content structure as
our Figure~\ref{fig:app-attnres-profile}. They also describe a PreNorm dilution effect: in a PreNorm
residual stream the hidden-state magnitudes grow with depth, so a later layer must emit an
ever-larger output to remain influential. This is a mechanism for the norm stratification we
observe. When the residual stream weights every contribution equally, magnitude is the only lever a
layer has for influence, and the keys inherit that inflation.

\paragraph{Projection sharing: the low-rank regime, and a collapse-to-SSM bridge.}
Asking whether all three attention projections are necessary, \citet{kayyam2026qkv} find the value
projection can be tied to the key (their Q-K=V variant) at a cost of about 3\% perplexity at 300M
and 2.5\% at 1.2B, and attribute the small cost to attention operating in a low-rank regime in which
the key and value occupy nearly the same subspace. Their weight-space measurements make this
concrete: across layers the trained key and value projections have cosine similarity $0.73$ and
nearly equal effective rank ($687$ versus $702$ of $1024$ dimensions), while the query stays
distinct (cosine $0.42$ with the key, $0.31$ with the value). This is the same low-rank collapse we
measure in the activation geometry, seen from the weight side: the keys are low-rank, and the values
have nearly merged into their subspace. Separately, the same paper derives a structural bridge we
did not. Under the full collapse $q=k=v$, a kernelized (linear) attention layer rewrites exactly as
a recurrent state-space update, $S_t = \lambda S_{t-1} + \phi(z_t) z_t^\top$ read out by the current
input, which is a state-space model with an input-conditioned rather than a fixed observation. It is
an idealization, linear attention rather than the softmax models we run, and it was reached from an
efficiency motive, but it is a formal echo of the continuum our measurements trace between attention
and state-space routing.

\paragraph{Memory caching: attention as cached recurrence.}
A second formal bridge runs the other way. \citet{behrouz2026memory} give recurrent models a growing
memory by caching checkpoints of the recurrent state, and show that the construction recovers gated
softmax attention as a special case, when each token is its own segment and the memory is
value-less, so that attention sits at one end of a complexity continuum whose other end is a
fixed-size recurrence. Where the projection-sharing result collapses attention into a recurrence,
this one expands a recurrence into attention. Both place the two on a single axis, which is the
cross-architecture reading of Section~\ref{sec:crossarch} reached by construction rather than by
measurement.

%% file: sections/appendix_taxonomy.tex
\section{A taxonomy of concentration phenomena}
\label{app:taxonomy}

The main text treats concentration as one mechanism with an architecture-dependent form. A
complementary question is where the concentration comes from in the first place. When reading
results across architectures it helps to separate two sources.

\paragraph{Structural concentration.}
Some concentration is inherited from fixed properties of the system rather than created by its
dynamics. The preference hierarchy exists before inference and is imposed by the substrate:
high-degree hubs in a graph, corpus-frequency imbalances in a language model, class imbalance in
a dataset, or a fixed architectural or retrieval prior. A diagnostic of structural concentration
is that removing the dominant element reveals the next element in a pre-existing hierarchy: the
ordering is largely fixed, and only the identity of the winner changes.

\paragraph{Emergent concentration.}
Other concentration is produced or amplified by the computational dynamics. Weak preferences are
reinforced through competition, routing, recurrence, or memory into dominant pathways that the
substrate did not hand the model. The four routers studied here, attention, selective state
spaces, recurrent mixing, and residual routing, all show concentration of this kind. Its
diagnostic is sensitivity to mechanism-level intervention: changing the routing changes which
source wins and how a winner emerges at all.

\paragraph{Concentration as a mixture.}
The distinction is not a binary. Most architectures sit somewhere on a spectrum, with a total
concentration that mixes the two sources,
\[ C_{\text{total}} = C_{\text{structural}} + C_{\text{emergent}}, \]
which we write schematically rather than as a measured decomposition. Equivalently, an
architecture can be placed in a plane spanned by a structural and an emergent axis, though the
precise placements remain an open empirical question. Graph attention carries a large structural
component from topology; the recurrent ablations isolate a large emergent component from learned
dynamics; transformer attention plausibly carries both. The question is then not whether
concentration is structural or emergent, but which source dominates in a given architecture and
how the two interact.

\paragraph{The controls in this work decompose the two.}
The taxonomy is more than interpretive here: two of our controls instantiate it. The GCN
control of Section~\ref{sec:crossarch} holds the graph fixed and removes learned attention, so its
collapse is the structural part, oversmoothing on the given topology; the GAT-minus-GCN
differential we report is then the emergent part, the concentration that learned attention adds on
top of the structure. So graph attention is not purely structural: the differential is exactly its
emergent component, separated out. The within-model ablations, freezing Mamba's selectivity and
sweeping RWKV's carry and AttnRes's recency, are direct probes of emergent concentration: each
intervenes on the mechanism and moves the concentration, the signature the taxonomy assigns to the
emergent source. The main results of this paper concern emergent concentration; the structural
component is held fixed or subtracted by these controls.

\paragraph{Scope.}
We offer this as an interpretive framework for organizing the observations, not as a proven
classification. A fuller account would measure the two components rather than reason about them,
and would extend the survey to families we do not test here, including probabilistic graphical
models, retrieval systems, and convolutional architectures. The broader hypothesis it suggests is
that concentration is a general property of systems that combine a structural bias with
competitive dynamics, and that the attractors we observe are particular mixtures of inherited
structure and emergent reinforcement.

%% file: references.bib
@article{jaynes1957,
  title={Information Theory and Statistical Mechanics},
  author={Jaynes, E. T.},
  journal={Physical Review},
  volume={106},
  number={4},
  pages={620--630},
  year={1957},
  publisher={American Physical Society}
}

@inproceedings{vaswani2017,
  title={Attention Is All You Need},
  author={Vaswani, Ashish and Shazeer, Noam and Parmar, Niki and Uszkoreit, Jakob
          and Jones, Llion and Gomez, Aidan N. and Kaiser, {\L}ukasz and Polosukhin, Illia},
  booktitle={Advances in Neural Information Processing Systems},
  volume={30},
  year={2017}
}

@inproceedings{li2018deeper,
  title={Deeper Insights into Graph Convolutional Networks for Semi-Supervised Learning},
  author={Li, Qimai and Han, Zhichao and Wu, Xiao-Ming},
  booktitle={AAAI Conference on Artificial Intelligence},
  year={2018}
}

@inproceedings{oono2020graph,
  title={Graph Neural Networks Exponentially Lose Expressive Power for Node Classification},
  author={Oono, Kenta and Suzuki, Taiji},
  booktitle={International Conference on Learning Representations},
  year={2020}
}

@article{ali2024hidden,
  title={The Hidden Attention of Mamba Models},
  author={Ali, Ameen and Zimerman, Itamar and Wolf, Lior},
  journal={arXiv preprint arXiv:2403.01590},
  year={2024}
}

@article{su2021roformer,
  title={RoFormer: Enhanced Transformer with Rotary Position Embedding},
  author={Su, Jianlin and Lu, Yu and Pan, Shengfeng and Murtadha, Ahmed and Wen, Bo and Liu, Yunfeng},
  journal={arXiv preprint arXiv:2104.09864},
  year={2021}
}

@inproceedings{press2022alibi,
  title={Train Short, Test Long: Attention with Linear Biases Enables Input Length Extrapolation},
  author={Press, Ofir and Smith, Noah A. and Lewis, Mike},
  booktitle={International Conference on Learning Representations},
  year={2022}
}

@inproceedings{shazeer2017moe,
  title={Outrageously Large Neural Networks: The Sparsely-Gated Mixture-of-Experts Layer},
  author={Shazeer, Noam and Mirhoseini, Azalia and Maziarz, Krzysztof and Davis, Andy
          and Le, Quoc and Hinton, Geoffrey and Dean, Jeff},
  booktitle={International Conference on Learning Representations},
  year={2017}
}

@article{kimi2026attnres,
  title={Attention Residuals},
  author={{Kimi Team}},
  journal={arXiv preprint arXiv:2603.15031},
  year={2026}
}

@inproceedings{velickovic2018gat,
  title={Graph Attention Networks},
  author={Veli{\v{c}}kovi{\'c}, Petar and Cucurull, Guillem and Casanova, Arantxa
          and Romero, Adriana and Li{\`o}, Pietro and Bengio, Yoshua},
  booktitle={International Conference on Learning Representations},
  year={2018}
}

@article{gu2023mamba,
  title={Mamba: Linear-Time Sequence Modeling with Selective State Spaces},
  author={Gu, Albert and Dao, Tri},
  journal={arXiv preprint arXiv:2312.00752},
  year={2023}
}

@inproceedings{peng2023rwkv,
  title={{RWKV}: Reinventing {RNN}s for the Transformer Era},
  author={Peng, Bo and Alcaide, Eric and Anthony, Quentin and others},
  booktitle={Findings of the Association for Computational Linguistics: EMNLP},
  year={2023}
}

@inproceedings{xiao2024streaming,
  title={Efficient Streaming Language Models with Attention Sinks},
  author={Xiao, Guangxuan and Tian, Yuandong and Chen, Beidi and Han, Song and Lewis, Mike},
  booktitle={International Conference on Learning Representations},
  year={2024}
}

@inproceedings{kayyam2026qkv,
  title={Do Transformers Need Three Projections? Systematic Study of {QKV} Variants},
  author={Kayyam, Ali and Madan Gopal, Anusha and Lewis, M. Anthony},
  booktitle={International Conference on Machine Learning},
  year={2026}
}

@article{behrouz2026memory,
  title={Memory Caching: {RNN}s with Growing Memory},
  author={Behrouz, Ali and Li, Zeman and Deng, Yuan and Zhong, Peilin and Razaviyayn, Meisam and Mirrokni, Vahab},
  journal={arXiv preprint arXiv:2602.24281},
  year={2026}
}

@inproceedings{hu2022lora,
  title={{LoRA}: Low-Rank Adaptation of Large Language Models},
  author={Hu, Edward J. and Shen, Yelong and Wallis, Phillip and Allen-Zhu, Zeyuan and Li, Yuanzhi
          and Wang, Shean and Wang, Lu and Chen, Weizhu},
  booktitle={International Conference on Learning Representations},
  year={2022}
}

@inproceedings{aghajanyan2021intrinsic,
  title={Intrinsic Dimensionality Explains the Effectiveness of Language Model Fine-Tuning},
  author={Aghajanyan, Armen and Gupta, Sonal and Zettlemoyer, Luke},
  booktitle={Annual Meeting of the Association for Computational Linguistics},
  year={2021}
}

@inproceedings{dong2021attention,
  title={Attention is Not All You Need: Pure Attention Loses Rank Doubly Exponentially with Depth},
  author={Dong, Yihe and Cordonnier, Jean-Baptiste and Loukas, Andreas},
  booktitle={International Conference on Machine Learning},
  year={2021}
}
